\documentclass{ecai} 

\usepackage{algorithm,algorithmic}
\usepackage{multirow}
\usepackage{rotating}

\usepackage{latexsym}
\usepackage{amssymb}

\usepackage{amsmath}
\usepackage[table]{xcolor}
\usepackage{multirow}
\usepackage{rotating}
\usepackage{amsthm}
\usepackage{booktabs}
\usepackage{enumitem}
\usepackage{graphicx}
\usepackage{color}

\newtheorem{theorem}{Theorem}
\newtheorem{lemma}[theorem]{Lemma}

\newtheorem{definition}{Definition}

\newcommand{\BibTeX}{B\kern-.05em{\sc i\kern-.025em b}\kern-.08em\TeX}

\begin{document}


\begin{frontmatter}


\paperid{2197} 


\title{TED: Accelerate Model Training by Internal Generalization}


\author[A]{\fnms{Jinying}~\snm{Xiao}\thanks{Corresponding Author. Email: xiaojinying1014@163.com}\footnote{Equal contribution.}}
\author[A]{\fnms{Ping}~\snm{Li}\footnotemark}
\author[A]{\fnms{Jie}~\snm{Nie}} 

\address[A]{Changsha University of Science and Technology}


\begin{abstract}
Large language models have demonstrated strong performance in recent years, but the high cost of training drives the need for efficient methods to compress dataset sizes. We propose TED pruning, a method that addresses the challenge of overfitting under high pruning ratios by quantifying the model's ability to improve performance on pruned data while fitting retained data, known as Internal Generalization (IG). TED uses an optimization objective based on Internal Generalization Distance (IGD), measuring changes in IG before and after pruning to align with true generalization performance and achieve implicit regularization. The IGD optimization objective was verified to allow the model to achieve the smallest upper bound on generalization error. The impact of small mask fluctuations on IG is studied through masks and Taylor approximation, and fast estimation of IGD is enabled. In analyzing continuous training dynamics, the prior effect of IGD is validated, and a progressive pruning strategy is proposed. Experiments on image classification, natural language understanding, and large language model fine-tuning show TED achieves lossless performance with 60-70\% of the data. Upon acceptance, our code will be made publicly available.
\end{abstract}

\end{frontmatter}


\section{Introduction}

Deep learning has achieved significant progress in both visual tasks \citep{ref19,ref20,ref21} and natural language tasks \citep{ref49,ref50,ref51}. Despite the impressive performance of models, most training and fine-tuning processes rely on ultra-large datasets, which results in significant computational demands and time costs. Reducing the training workload on large datasets is essential for the broader application of deep learning.

Dataset pruning aims to accelerate training while maintaining model performance by retaining high-performance subsets or reducing the number of iterations. Previous work focused on deleting redundant samples to retain a compact core set \citep{ref22,ref23,ref24,ref28}. However, these methods often require training one or more surrogate models to obtain data features and distribution, resulting in longer overhead than a single training cycle. Moreover, they lacks dynamic awareness, as these methods tend to retain samples that perform well in later stages of training \citep{ref43,ref44}.

To address these issues, recent methods have proposed dynamic dataset pruning during training \citep{ref5,ref6}. These approaches reconceptualize dataset pruning as a dynamic decision-making process that prunes the entire dataset during each cycle, allowing previously pruned samples to re-enter subsequent training. By closely tying dynamic scoring to model trajectories, these methods have achieved some success.

Regarding pruning criteria, many methods design evaluation functions to identify redundant samples. The most effective approach is pruning based on sample importance \citep{ref1,ref5,ref6,ref27,ref28}, either through logits and labels \citep{ref5,ref6,ref27} or gradients \citep{ref1,ref27,ref28}. Logit-based methods assess data based on model loss and retain samples that are difficult to learn, while gradient-based methods use parameter gradients to reflect sample importance.

\begin{figure}[H]
	\centering
	\includegraphics[width=5.5cm]{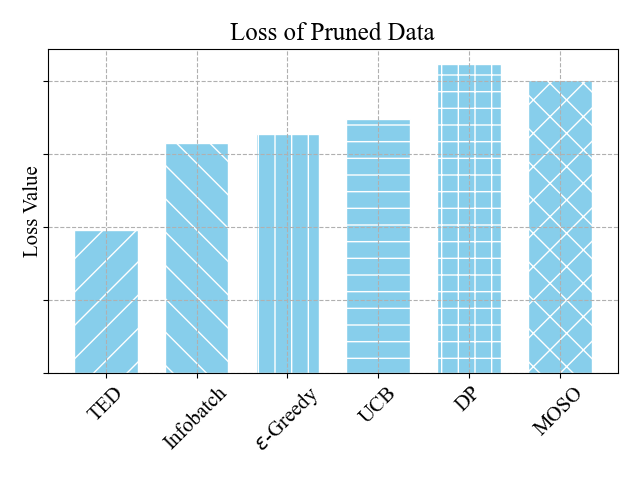}
	\caption{We used CIFAR100-ResNet18 to evaluate model loss on pruned data after convergence at a 70\% pruning ratio under different pruning methods. In dynamic pruning, we collected the data that was least frequently selected at each checkpoint.}
	\label{fig1}
\end{figure}
 
\textbf{Limitations and Motivations.} Although these approaches have yielded some success, there are still limitations that need to be addressed. (i) Overfitting of Retained Data: As illustrated in Figure~\ref{fig1}, high pruning ratios lead to poor fitting performance on the pruned data after model convergence. This suggests a disconnect between fitting the retained data and generalization to pruned data, as the pruned data do not participate in training. It is believed that the excessive focus of these methods on the retained data makes the model prone to overfitting, preventing it from gaining a comprehensive understanding from a small dataset. Essentially, their scoring functions are not efficient and do not retain data with intrinsical characteristics. (ii) Static Pruning Ratio: Hard and soft pruning \citep{ref5,ref6} were proposed, but the amount of data the model requires varies at different stages. They do not consider the dynamic changes of the optimal pruning rate during training, which can lead to a decrease in performance and an increase in training costs.

\textbf{Our contributions are as follows:}
﻿\begin{itemize}
\item We identified Internal Generalization (IG) in dataset pruning, where fitting the retained data also enhances the model's generalization to pruned data. IG was quantified in continuous training dynamics, and an optimization objective based on Internal Generalization Distance (IGD) before and after pruning was proposed. IGD's close coupling with model generalization performance led us to frame dataset pruning as optimizing IGD. This approach not only directly reflects model generalization but also achieves implicit regularization of the model on pruned data. To this end, the IGD optimization objective was verified to have the smallest upper bound on generalization error. By adding masks to data and using Taylor approximation, we evaluate the importance of samples through small changes in the masks to achieve rapid estimation of IGD.
\item While IGD is typically derived from post-training knowledge, our analysis of continuous training dynamics shows that IGD has a prior effect throughout the training cycle. This allows us to use IGD as an effective optimization objective from the early stages of training without the need to train a surrogate model. Based on experimental and theoretical analysis, a progressive data pruning strategy is proposed that schedules the pruning ratio from low to high. This approach minimizes training time while maintaining performance, thanks to the carefully designed pruning ratio schedule.
\item We introduced TED (In\underline{\textbf{t}}ernal G\underline{\textbf{e}}neralization \underline{\textbf{D}}istance) pruning, a dynamic data pruning approach. TED's performance was evaluated on image classification, natural language understanding (NLU), and large language model (LLM) fine-tuning tasks. In most cases, TED achieved lossless performance with 60\%-70\% of the data and significantly outperformed prior state-of-the-art methods, particularly at high pruning ratios.
\end{itemize}

\section{Related Work}
\textbf{Static Pruning.} Static pruning aims to select a compact core subset of data before training. For example, \cite{ref11} eliminates memorable samples based on their forgetfulness during training; \cite{ref45} selects samples with maximum diversity in gradient space; \cite{ref22} chooses a moderate core set based on the distance from data points to the center; \cite{ref1} uses influence functions to construct a minimal subset based on specific samples' impact on the model's generalization ability; and \cite{ref28} evaluates samples' importance on empirical risk and approximates scoring with a first-order approximator. While these methods select an efficient core subset, the core subset represents post-training knowledge and often requires the use of one or more surrogate models to acquire data characteristics and distribution, as in \cite{ref1}, where SENet and ResNet18 are used as surrogates to accelerate ResNet50 training. This approach leads to additional overhead and may result in a core subset lacking generalizability, such as in \cite{ref1}, where different surrogate models are needed for training ResNet architectures of varying depths.

\textbf{Dynamic Pruning.} Dynamic pruning aims to address additional overhead by aligning optimal dynamic scoring closely with the model's training trajectory. For instance, \cite{ref6} categorizes data into three groups based on the frequency of sample selection and finds that samples can transition significantly within training dynamics. They perform dataset pruning at each checkpoint based on the loss of retained data during training without using surrogate networks. Their conclusions suggest that dynamic pruning consistently outperforms static pruning, and even random selection. \cite{ref5} introduced soft pruning, arguing that while hard pruning with a constant pruning ratio requires a complexity of $\mathcal{O}(logN)$ for ranking, soft pruning only requires $\mathcal{O}(1)$. Despite reducing ranking cost, soft pruning struggles to determine the actual pruning ratio, which increases model training costs. Notably, prior to our work, \cite{ref5} and \cite{ref6} represented the most advanced dynamic pruning performance.

\section{Method}

\subsection{Problem Definition}

Given a dataset $\mathcal{D} = \{z_1, \ldots, z_n\}$ containing $n$ training samples, the goal of dataset pruning is to identify a redundant subset $\hat{\mathcal{D}}$ from the original samples, where $\hat{\mathcal{D}}\subset \mathcal{D}$, to accelerate the training process of the model. We can express the optimization parameter through $\mathcal{D} - \hat{\mathcal{D}}$ as $\hat{\theta}_{-\hat{\mathcal{D}}} =  \arg\min_{\theta} L(\mathcal{D} - \hat{\mathcal{D}}, \theta)$, where $L(\mathcal{D} - \hat{\mathcal{D}}, \theta)= \mathbb{E}_{z \in \mathcal{D} - \hat{\mathcal{D}}} loss(z, \theta)$, and $loss$ is the loss function, which is cross-entropy loss for classification tasks.
\subsection{Internal Generalization}\label{sub3.2}

Recent dynamic pruning methods \citep{ref5,ref6} often use loss as the sample scoring metric, prioritizing samples that are difficult to learn. This approach tends to choose higher-loss samples at each checkpoint, aiming to advance them to the next epoch. However, as these samples' scores (loss) decrease, they become less likely to be selected at future checkpoints. This method prioritizes extensive learning over capturing essential data features, resulting in a performance that is often similar to random dynamic pruning, as shown in subsequent experiments.

Therefore, the aim is to identify samples that can improve the model's generalization performance. Internal generalization (IG) refers to the model's ability to implicitly reduce loss among data during the training process. Specifically, the model enhances its performance on pruned data while fitting the retained data, even though the pruned data do not participate in training, as shown in Figure~\ref{fig2}. This observation demonstrates the model's generalization ability on data not involved in training. IG is a key characteristic of the model's performance on pruned data and indirectly reflects the model's true generalization performance across the entire dataset. By capturing IG, we can gain insights into the model's ability to generalize beyond the training data.
\begin{figure}[H]
	\centering
	\includegraphics[width=8cm]{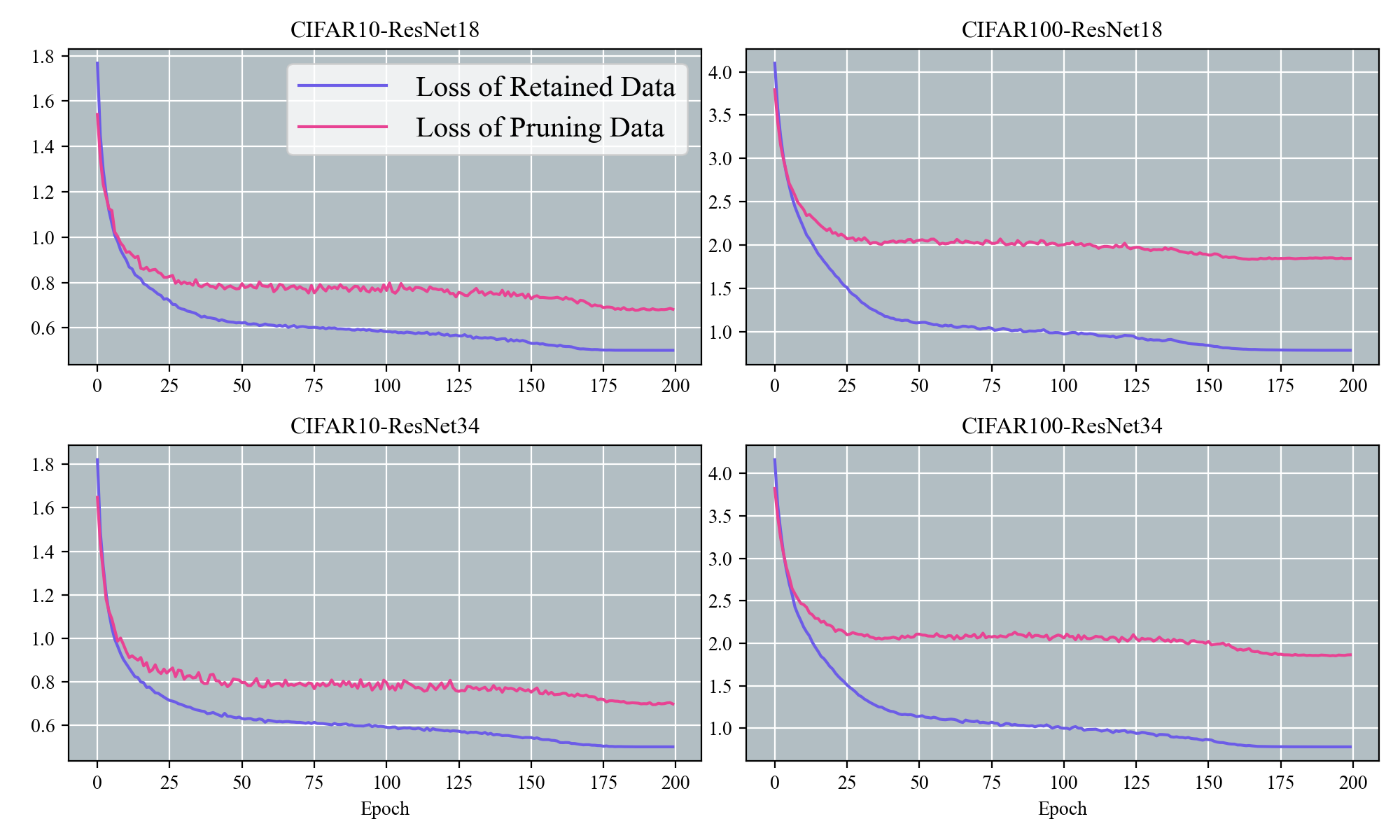}
	\caption{The variation in loss for retained data and pruned data across different datasets and architectures using 60\% static random pruning is depicted in the figure.}
	\label{fig2}
\end{figure}

We express IG formally. Let $\{\theta_1, \theta_2, \ldots, \theta_{T-1}, \theta_T\}$ be the iterations under gradient descent (GD) using data $\mathcal{D} - \hat{\mathcal{D}}$, where $\theta_t \in \mathbb{R}^{1 \times P}$. For ease of analysis, each iteration is assumed to be done using the full batch, which can be represented as:
\begin{eqnarray}\label{eq1}
	\theta_t=\theta_{t-1}-\eta\nabla_{\theta_{t-1}}L(\mathcal{D}-\hat{\mathcal{D}},\theta_{t-1})
\end{eqnarray}

Where $\nabla_{\theta_{t-1}}L(\mathcal{D}-\hat{\mathcal{D}},\theta_{t-1})\in \mathbb{R}^{1 \times P}$,$\eta$ represents the learning rate. Additionally, the training dynamics are approximated as if they are in continuous time, a method also used in other research \citep{ref7, ref8}. When \(n \gg m\), the change in \(L(\mathcal{D}, \theta_t)\) at time \(t\) is defined, and according to the chain rule, we can obtain:
\begin{eqnarray}\label{eq2}
	\frac{dL(\mathcal{D},\theta_t)}{dt}\approx-\eta\ \nabla_{\theta_t}L(\mathcal{D}-\hat{\mathcal{D}},\theta_t)\nabla_{\theta_t}L(\mathcal{D}-\hat{\mathcal{D}},\theta_t)^T
\end{eqnarray}

Based on the iteration in Equation~(\ref{eq1}), the model reduces $L(\mathcal{D}, \theta_t)$, which includes fitting the retained data in $\mathcal{D} - \hat{\mathcal{D}}$ and generalizing to pruned data in $\hat{\mathcal{D}}$. In simpler terms, the model learns the properties of the pruned data through its training on the retained data. This process enables the model to generalize effectively to data it has never seen during training, capturing the characteristics of pruned data and demonstrating its broader generalization ability.
Therefore, we express the cumulative change in \(\frac{dL(\mathcal{D}, \theta_t)}{dt}\) due to fitting with \(\mathcal{D} - \hat{\mathcal{D}}\) as:
\begin{eqnarray}\label{eq3}
	IG=\int_{1}^{T}\frac{dL(\mathcal{D},\theta_t)}{dt}dt=L(\mathcal{D},\theta_T)-L(\mathcal{D},\theta_1)
\end{eqnarray}


\begin{definition}[Internal Generalization Distance, IGD]\label{def1}
Given a dataset \(\mathcal{D}\), let \(\hat{\mathcal{D}}\) be a subset of \(\mathcal{D}\), represented as \(\hat{\mathcal{D}} \subset \mathcal{D}\). The parameters optimized by the two sets of data, \(\mathcal{D} - \hat{\mathcal{D}}\) and \(\mathcal{D}\), are denoted as \(\hat{\theta}_{-\hat{\mathcal{D}}} = \arg\min_\theta L(\mathcal{D}_{-\hat{\mathcal{D}}}, \theta)\) and \(\hat{\theta} = \arg\min_\theta L(\mathcal{D}, \theta)\) respectively. Based on Equation~(\ref{eq3}), we define the internal generalization distance from \(\hat{\theta}_{-\hat{\mathcal{D}}}\) to \(\hat{\theta}\):
\begin{eqnarray}\label{eq4}
	IGD_{-\hat{\mathcal{D}}}=L(\mathcal{D},\hat{\theta}_{-\hat{\mathcal{D}}})-L(\mathcal{D},\hat{\theta})
\end{eqnarray}
\end{definition}
According to Equation~(\ref{eq4}), both \(\hat{\theta}\) and \(\hat{\theta}_{-\hat{\mathcal{D}}}\) perfectly fit \(\mathcal{D} - \hat{\mathcal{D}}\), so IGD primarily reflects the performance of \(\hat{\theta}\) and \(\hat{\theta}_{-\hat{\mathcal{D}}}\) on \(\hat{\mathcal{D}}\). IGD characterizes the model's performance on $\mathcal{D}$ and forms a close relationship between the IG of the dataset and the model's true generalization performance. This approach achieves loss reduction for $\mathcal{D}$ and necessitates that the model possess a certain level of IG by focusing on the whole dataset. This process implicitly regularizes training, promoting more stable and effective model learning.

\begin{lemma}\label{lem1}
	Given a dataset \(\mathcal{D}\) and a pruning ratio \(s\), if a subset \(\mathcal{D}_k\) of \(\mathcal{D}\) satisfies:
	\begin{eqnarray}\label{eq5}
		\mathcal{D}_k=\arg\max_{\mathcal{D}_i}IGD_{-\mathcal{D}_i}
	\end{eqnarray}

And $s=\frac{\Vert\mathcal{D}-\mathcal{D}_k\Vert_0}{\Vert \mathcal{D} \Vert_0}$, then the parameters $\hat{\theta}_{\mathcal{D}_k}$ have the smallest upper bound on generalization error.
\end{lemma}
\begin{proof}\label{key}
	It is assumed that the parameters \(\hat{\theta}_{\mathcal{D}_k}\) have the largest upper bound on generalization error. Previous work \cite{ref52} expressed the upper bound on generalization error, which in our case can be represented as:
	\begin{eqnarray}\label{pr_eq1}
		\mathcal{R}(\hat{\theta})\le\hat{\mathcal{R}}(\mathcal{D},{\hat{\theta}}_{\mathcal{D}_k})+\varepsilon
	\end{eqnarray}

Where \(\mathcal{R}(\hat{\theta})\) is the expected loss, and \(\hat{\mathcal{R}}(\mathcal{D}, {\hat{\theta}}_{\mathcal{D}_k})\) is the empirical risk of \({\hat{\theta}}_{\mathcal{D}_k}\), which is our fitting loss. \(\varepsilon\) is a coefficient related to the model size and the size of the retained dataset. When \(\hat{\theta}_{D_k}\) has the largest upper bound on generalization error, then \(\hat{\mathcal{R}}(\mathcal{D}, {\hat{\theta}}_{{\mathcal{D}}_k})\) is also at its maximum. It can be expressed as:
\begin{eqnarray}\label{pr_eq2}
	\mathcal{D}_k=\arg\max_{\mathcal{D}_i}L(\mathcal{D},\hat{\theta}_{\mathcal{D}_i})-L(\mathcal{D},\hat{\theta})
\end{eqnarray}

Based on empirical observations, since \(\hat{\mathcal{D}}_k\) and \(\mathcal{D} - \hat{\mathcal{D}}_k\) are mutually exclusive, Equation~\ref{pr_eq2} can be expressed as:
\begin{eqnarray}
	\mathcal{D}_k=\arg\min_{\mathcal{D}_i}L(\mathcal{D},\hat{\theta}_{-\mathcal{D}_i})-L(\mathcal{D},\hat{\theta})
\end{eqnarray}

Therefore, based on the rule of contraposition, Lemma~\ref{lem1} can be proven.
\end{proof}
Curriculum learning \cite{ref44} demonstrates that samples can be ranked, allowing us to define redundant sample based on IGD.
\begin{definition}[Redundant sample]\label{def2}
	Given a dataset \(\mathcal{D}\), and a sample \(z_k\) from \(\mathcal{D}\), represented as \(z_k \in \mathcal{D}\), if \(z_k\) satisfies:
	\begin{eqnarray}\label{eq6}
		z_k=\arg\min_{z_i}IGD_{-z_i}
	\end{eqnarray}

Then we call \(z_k\) a redundant sample of \(\mathcal{D}\).
\end{definition}
As seen from Equation~(\ref{eq6}), the process of selecting \(z_k\) is not straightforward. It involves a long-term optimization process for two networks, which is costly. Specifically, for a dataset of \(n\) samples, calculating \(IGD_{-z_k}\) for each sample has a complexity of \(\mathcal{O}(nT)\), where \(T\) is the number of training iterations required to obtain \(\hat{\theta}_{-z_k}\). For large-scale networks and datasets with tens of thousands of samples, this overhead is significant and cannot be ignored.
To solve this problem, a mask was added for each sample. Equation~(\ref{eq4}) can be expressed as:
\begin{eqnarray}\label{eq7}
	IGD_{-z_k}=L(\delta_k=0)-L(\delta_k=1)
\end{eqnarray}

Where \(\delta_k \in \{0,1\}\), \(L(\delta_k) = L(\mathcal{D}, \arg\min_\theta L(\mathcal{D}_{-z_k} + \delta_i z_k, \theta))\).

According to Equation~(\ref{eq7}), \(IGD_{-z_k}\) depicts the change in IG when the mask is set to 0 or 1. In the model's forward propagation, we can assume \(\delta_k\) is no longer a discrete value, a similar assumption to that used in \cite{ref4}.

For a sample \(z_k\), the impact of \(z_k\) on IG can be approximated using Taylor expansion:
\begin{eqnarray}\label{eq8}
	L(\delta_k=0)-L(\delta_k=1)\approx-\frac{\partial L(\delta_k)}{\partial\delta_k}|_{\delta_k=1}
\end{eqnarray}

Therefore, the change in the mask \(\delta_k\), i.e., the impact of removing \(z_k\) on IG, is represented as:
\begin{eqnarray}\label{eq9}
	\begin{split}
	\mathcal{I}(z_k) =&\left|\frac{\partial L(\delta_k)}{\partial\delta_k}|_{\delta_k=1}\right|\\
	=&\left|\lim_{\lambda \rightarrow0}{\frac{L(\delta_k=1)-L(\delta_k=1+\lambda)}{\lambda}|_{\delta_k=1}}\right|
	\end{split}
\end{eqnarray}

As seen from Equation~(\ref{eq9}), rapid estimation of IGD is achieved through small changes in \(\delta_k\).

\subsection{Dynamic Pruning}\label{sub3.3}
According to Equation~(\ref{eq9}), IGD serves as a form of post-training knowledge and requires optimizing the loss function \(L(\mathcal{D}_{-z_k}, \theta)\) to identify redundant samples. Our aim is to dynamically identify these redundant samples early or during training to accelerate the model's training process. By doing so, training can be streamlined and efficiency improved while maintaining model performance.

Following the continuous training dynamics described in Section~\ref{sub3.2}, under the optimization of \(\mathcal{D}-z_k\), the IGD at time \(t\) is:
\begin{eqnarray}\label{eq10}
	IGD_{-z_k}^{(t)}=L(\mathcal{D},\theta_t)-L(\mathcal{D},\hat{\theta})
\end{eqnarray}

Where \(\hat{\theta} = \arg\min_\theta L(\mathcal{D}, \theta)\), and \(\theta_t\) is trained using data \(\mathcal{D} - z_k\). We represent the instantaneous ratio of change of \(IGD_{-z_k}^{(t)}\) at time \(t\):
\begin{eqnarray}\label{eq11}
	\mathcal{V}_t=\frac{d(IG D_{-z_k}^{(t)})}{dt}|_{t=t}
\end{eqnarray}
\begin{lemma}\label{lem2}
	For early training time \(t\) and the end of training time \(T\), there exists a constant \(C\) such that:
	\begin{eqnarray}\label{eq12}
		\left|\mathcal{V}_t-\mathcal{V}_T\right|\le\left|\nabla_{\theta_t}L(\mathcal{D},\theta_t)\right|C\
	\end{eqnarray}
	For different optimization methods, the value of \(C\) varies.
\end{lemma}
\begin{proof}
	\(\mathcal{V}_t - \mathcal{V}_T\) can be expressed as:
	\begin{eqnarray}\label{pr_eq3}
		\begin{split}
		\mathcal{V}_t-\mathcal{V}_T=&\eta \nabla_{\theta_t}L(\mathcal{D},\theta_t)(\frac{d\theta_t}{dt})^T|_{t=T}\\
		-&\eta \nabla_{\theta_t}L(\mathcal{D},\theta_t)(\frac{d\theta_t}{dt})^T|_{t=t}
	\end{split}
	\end{eqnarray}

At time \(T\), which is the end of training and the point at which the model has converged, Equation~\ref{pr_eq3} can be rewritten as:
\begin{eqnarray}
	\mathcal{V}_t-\mathcal{V}_T=-\eta \nabla_{\theta_t}L(\mathcal{D},\theta_t)(\frac{d\theta_t}{dt})^T|_{t=t}
\end{eqnarray}

Therefore, letting \(C = \eta\vert \frac{d\theta_t}{dt}\vert\), Lemma~\ref{lem2} can be derived.
\end{proof}

Based on Lemma~\ref{lem2}, removing sample \(z_k\) in the early stages of training has a bounded impact on the change in \(IGD_{-z_k}^{(t)}\) later in training. This implies that the optimization objective of IGD has a prior effect. Consequently, dynamic pruning can be implemented using our scoring method. Equation~(\ref{eq8}) can be expressed as:
\begin{eqnarray}\label{eq13}
	\mathcal{I}_{(t)}(z_k)=\left|\frac{\partial L(\mathcal{D},\theta_t)}{\partial\delta_k}|_{\delta_k=1}\right|
\end{eqnarray}

In this work, our sample scoring is obtained through multiple pruning rounds. The calculation of mask gradients can be quite noisy \cite{ref9}, particularly in transformer models \cite{ref10}, so we use Polyak averaging:
\begin{eqnarray}\label{eq14}
	score^{(t)}(z_k)=\alpha\ast score^{(t-1)}(z_k)+(1-\alpha)\ast\mathcal{I}_{(t)}(z_k)
\end{eqnarray}
\begin{definition}[Dynamic redundant subset]\label{def3}
	Given a pruning ratio \(s_t\) at time \(t\), an original dataset \(\mathcal{D} = \{z_1, \ldots, z_n\}\), and a subset $\hat{\mathcal{D}} = \{\hat{z}_1, \ldots, \hat{z}_m\}$ of $\mathcal{D}$, and model parameters \(\theta_t\), if the following conditions are met:
	\begin{eqnarray}
		\forall\widehat{z_k}\in\hat{\mathcal{D}},\forall z_i\in\mathcal{D}-\hat{\mathcal{D}},score^{(t)}(\hat{z}_k)\le score^{(t)}(z_i)
	\end{eqnarray}

At time \(t\), \(\hat{\mathcal{D}}\) is called a redundant subset of \(\mathcal{D}\), and \(\mathcal{D} - \hat{\mathcal{D}}\) is called a non-redundant subset of \(\mathcal{D}\).
\end{definition}

\subsection{Pruning Ratio Scheduling}
Unlike other work involving soft and hard pruning \citep{ref5}, the aim is to optimize the approximation in Equation~(\ref{eq9}) by scheduling the pruning ratio. It was found that using progressive pruning can save more resources while maintaining performance.

To avoid confusion with symbols, the model parameters were denoted as \(w\), and \(\hat{w}\) as the optimized parameters. An experiment was conducted in which a certain amount of random pruning was performed at the \(t\) epoch during the entire training process, using a pruning ratio of \(s\) (the number of pruning sessions is typically 10\% of the total number of epochs). The optimized parameters in this scenario were expressed as \(\hat{w}_t^s\). As depicted in Figure~\ref{fig3}, the change in \(L(\mathcal{D}, \hat{w}_t^s)\) over \(t\) was reported across various model architectures. It was found that \(\frac{dL(\mathcal{D}, \hat{w}_t^s)}{dt} > 0\), and from previous work \citep{ref5, ref6}, \(\frac{dL(\mathcal{D}, \hat{w}_t^s)}{ds} > 0\). It can be deduced that \(\frac{ds}{dt} > 0\). In simple terms, during the entire model optimization process, the pruning ratio \(s\) should gradually increase as the pruning time \(t\) progresses in order to align with training dynamics.
\begin{figure}[H]
	\centering
	\includegraphics[width=8cm]{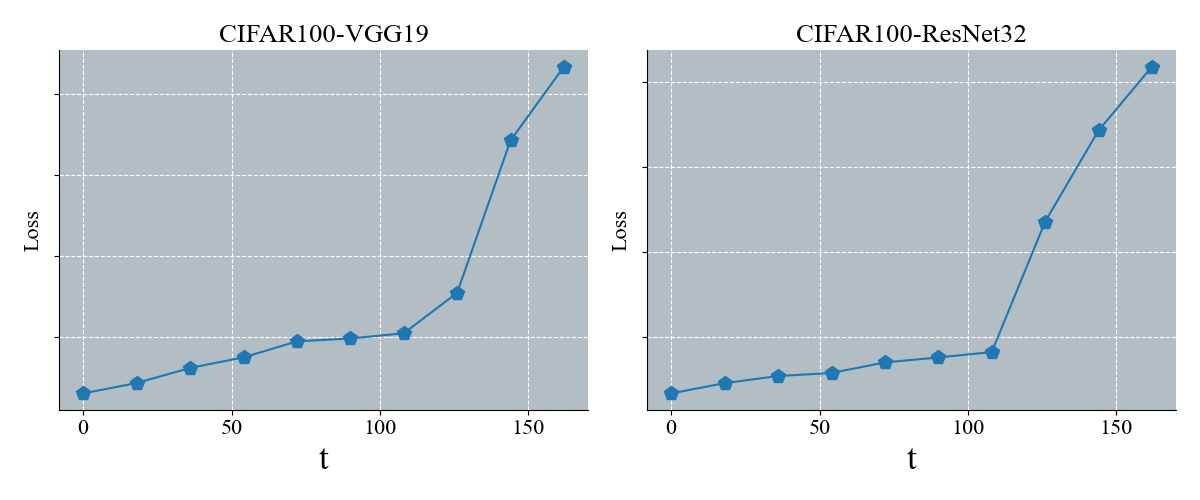}
	\caption{The graph shows the change in loss when training the VGG and ResNet models multiple times, performing 18 pruning sessions in each training run starting from epoch $t$.}
	\label{fig3}
\end{figure}

A simple exponential decay schedule was discovered and is called the "roller-coaster" schedule. It is defined as follows:
\begin{eqnarray}\label{eq15}
	r_t=exp\{(t/T)^\beta log(1-s)\}
\end{eqnarray}

Where \(s\) represents the pruning ratio, \(r_t\) represents the current data retention ratio, and \(\beta\) is a controllable hyperparameter. Small pruning steps enable reliable utilization of the approximation in section 3.2. Our schedule is a fast-to-slow approach, specifically with higher pruning ratios in the early stages and slower ratios in the later stages of training. This phenomenon intensifies as \(\beta\) decreases. The specific changes in pruning ratios can be seen in Appendix B of the supplementary materials.

A fast-to-slow pruning schedule is believed to be capable of removing a large number of samples without loss during the early stages of training. In the later stages, however, pruning boundaries may change drastically \cite{ref6}, necessitating a slower pruning ratio at that time.
\subsection{Unbiased Loss}
When training the network using pruned data, there can be an issue of insufficient iterations. According to \cite{ref5}, the gradient direction of the retained data is not consistent with that of the original data. Formally, we express this as:
\begin{eqnarray}\label{eq16}
	\nabla_{\theta_{-\mathcal{\hat{\mathcal{D}}}}}L(\mathcal{D},\theta_{-\mathcal{\hat{\mathcal{D}}}})\neq c_1\nabla_{\theta}L(\mathcal{D},\theta)
\end{eqnarray}

Where \( c_1 \) is a scalar.

Our report refutes this conclusion. As shown in Figure~\ref{fig4}, the Pearson correlation coefficient (PCC) of the two gradients was examined at different stages of training. It is evident that in the early stages of training, the two gradients are highly linearly correlated. Therefore, it is believed that directly applying a linear transformation to the loss in the early stages of training can facilitate effective gradient descent, expressed as:
\begin{eqnarray}\label{eq17}
	\hat{L}(\mathcal{D}-\hat{\mathcal{D}},\theta)=L(\mathcal{D}-\hat{\mathcal{D}},\theta)/r_t
\end{eqnarray}
\begin{figure}[H]
	\centering
	\includegraphics[width=7cm]{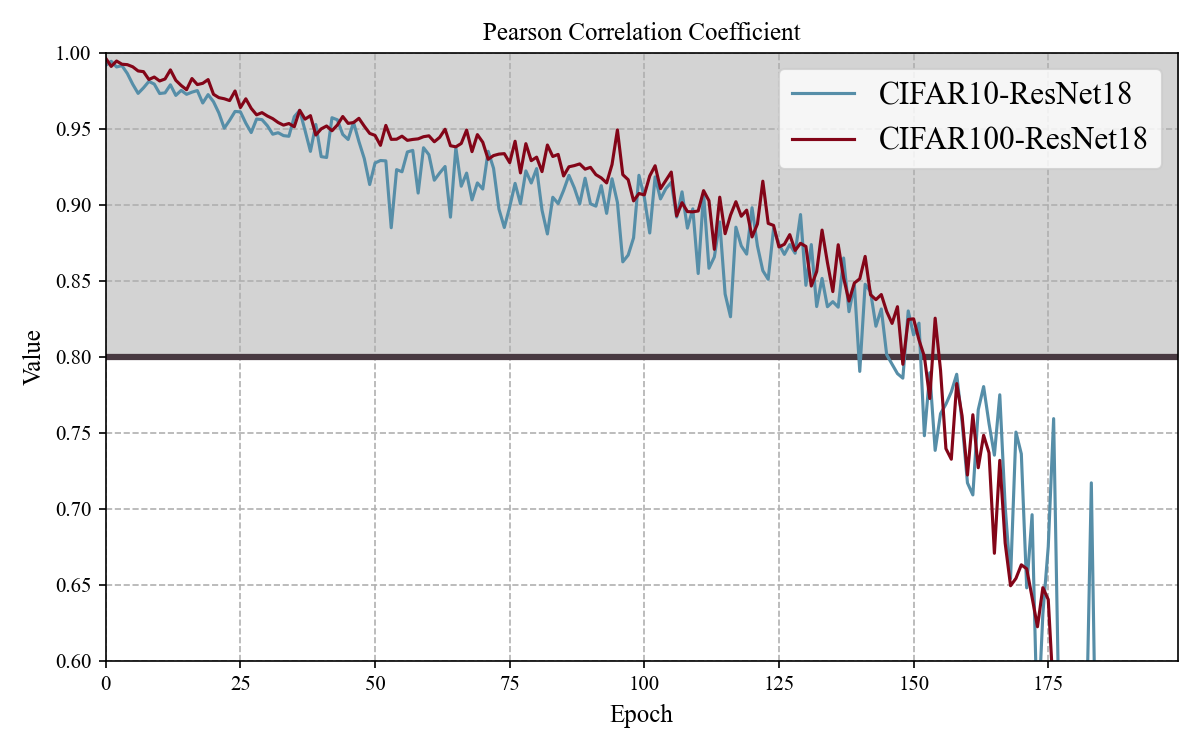}
	\caption{Using our method, the pearson correlation coefficient values of the gradients between retained data and full dataset across different datasets were calculated.}
	\label{fig4}
\end{figure}

The TED dataset pruning method is described in Algorithm 1.
 \begin{algorithm}[H]
	\caption{TED Data Pruning Process}
	
	\label{alorithm pre}
	\begin{algorithmic}[1]
		\REQUIRE Dataset $\mathcal{D}$ = $\{{z_1},...,{z_n}\}$, Epoch $T$, Pruning Ratio $s$, Initialized Model $\theta$, Initialized Redundant Dataset $\hat{\mathcal{D}}_0=\emptyset$.
		\hrule
		\vspace{0.1cm}
		\FOR{$t$ from 0 to $T-1$}
		\STATE Calculate keep ratio $r_t$ with Equation~(\ref{eq15})
		\STATE Get train sequence ${\mathcal{D}_t}=\mathcal{D}-\hat{\mathcal{D}}_t=\{B_0,B_1,,,B_{b_t-1}\}$
		\FOR{$j$ from 0 to $b_t-1$}
		\STATE Get loss $L(B_j,\theta_t^j)$
		\STATE Correct loss with Equation~(\ref{eq17}): $\hat{L}(B_j,\theta_t^j)$
		\STATE Update score with Equation~(\ref{eq14})
		\STATE Update model with optimizer
		\ENDFOR
		\STATE Get pruned dataset $\hat{\mathcal{D}}_{t+1}$
		
		\ENDFOR
		\vspace{0.1cm}
	\end{algorithmic}
\end{algorithm}
\section{Experiments}
In the following sections, the effectiveness of the theoretical results and the proposed dataset pruning method are validated through experiments. In Section 4.1, TED is compared with several SOTA methods on image classification tasks. In Section 4.2, the efficient performance of TED is demonstrated on natural language understanding tasks. In Section 4.3, pre-trained LLMs are fine-tuned using a low-rank adaptation method and compared with other methods. In Section 4.4, the effectiveness of the proposed theoretical results is verified through a series of ablation experiments. In Section 4.5, TED's generalization performance is analyzed using a one-dimensional linear interpolation method. To eliminate the impact of randomness, each experiment was run five times and the average taken. Implementation details can be found in Appendix A of the supplementary materials.

\subsection{Image Classification Task}
\textbf{CIFAR.} We conducted comparisons on the CIFAR-10 and CIFAR-100 datasets, presenting recent methods from the past few years, with the results shown in Table~\ref{table1_cifar10} and Table~\ref{table1_cifar100}. TED achieved lossless performance on CIFAR-100, even outperforming the baseline slightly. Notably, TED significantly surpassed previous static pruning methods and led other dynamic pruning methods at high pruning ratios. Additionally, Figure~\ref{fig4} and \ref{fig5} show TED's variation curves under different pruning ratios. Surprisingly, TED did not perform as well as expected at low pruning ratios. It is believed that the effect of IG is not significant at low pruning ratios, and the small size of the pruned dataset accentuates a small portion of noise \cite{ref21} and outliers \cite{ref33}, making a focus on the pruned data less efficient. However, as the pruning ratio increases, TED nearly matches the performance of the baseline and is more efficient than other methods.

\begin{table}[b]
	\begin{center}
	{\caption{The comparison results of the TED method with other methods on the CIFAR-10 dataset, using the ResNet-18 model. Pruning methods are categorized into static pruning and dynamic pruning. "Random*" refers to random dynamic dataset pruning. "Baseline" refers to the performance without pruning.}}
	\label{table1_cifar10}
	\resizebox{\columnwidth}{!}{
		\begin{tabular}{ccccc}
			\toprule
			& Dataset & \multicolumn{3}{c}{CIFAR10} \\
			\cmidrule(lr){1-2} \cmidrule(lr){3-5}
			&	Pruning Ratio \% & 30 & 50 & 70 \\
			\cmidrule(lr){1-2} \cmidrule(lr){3-5} 
			\multirow{11}{*}{\begin{sideways}Static\end{sideways}}
			&Random & 94.6 & 93.3 & 90.2  \\
			&CD\cite{ref13} & 95.0 & 94.3 & 90.8  \\
			&Least Confidence\cite{ref29} & 95.0 & 94.5 & 90.3  \\
			&Margin\cite{ref29} & 94.9 & 94.3 & 90.9  \\
			
			&GraNd-4\cite{ref27} & 95.3 & 94.6 & 91.2  \\
			&DeepFool\cite{ref38} & 95.1 & 94.1 & 90.0  \\
			&Craig\cite{ref52} & 94.8 & 93.3 & 88.4  \\
			&Glister\cite{ref12}  & 95.2 & 94.0 & 90.9  \\

			&EL2N-20\cite{ref11} & 95.3 & 95.1 & 91.9  \\
			&DP\cite{ref1} & 94.9 & 93.8 & 90.8  \\
			\midrule
			\multirow{5}{*}{\begin{sideways}Dynamic\end{sideways}}
			&Random* & 94.8 & 94.5 & 93.0  \\
			&$\epsilon$-greedy\cite{ref6}  & 95.2 & 94.9 & 94.1  \\
			&UCB\cite{ref6} & 95.3 & 94.7 & 93.9  \\
			&InfoBatch\cite{ref5} & \textbf{95.6} & 95.1 & 94.7 \\
			&TED(ours) &95.51\textsubscript{$\pm0.14$}  &\textbf{95.31}\textsubscript{$\pm0.09$}  &\textbf{94.94}\textsubscript{$\pm0.19$}    \\
			\midrule
			&Baseline & \multicolumn{3}{c}{95.6$\pm0.1$} \\				
			\bottomrule
		\end{tabular}
	}
	\end{center}
\end{table}

\begin{table}[b]
	\begin{center}
	\caption{The comparison results of the TED method with other methods on the CIFAR-100 dataset, using the ResNet-18 model. Pruning methods are categorized into static pruning and dynamic pruning. "Random*" refers to random dynamic dataset pruning. "Baseline" refers to the performance without pruning.}
	\label{table1_cifar100}
	\resizebox{\columnwidth}{!}{
		\begin{tabular}{ccccc}
			\toprule
			& Dataset  & \multicolumn{3}{c}{CIFAR100} \\
			\cmidrule(lr){1-2} \cmidrule(lr){3-5} 
			&	Pruning Ratio \%  & 30 & 50 & 70 \\
			\cmidrule(lr){1-2} \cmidrule(lr){3-5}
			\multirow{11}{*}{\begin{sideways}Static\end{sideways}}
			&Random & 73.8 & 72.1 & 69.7 \\
			&CD\cite{ref13}  & 74.2 & 72.3 & 70.3 \\

			&Least Confidence\cite{ref29}  & 74.2 & 72.3 & 69.8 \\
			&Margin\cite{ref29}  & 74.0 & 72.2 & 70.2 \\
			
			&GraNd-4\cite{ref27} & 74.6 & 71.4 & 68.8 \\
			&DeepFool\cite{ref38}  & 74.2 & 73.2 & 69.8 \\
			&Craig\cite{ref52}  & 74.4 & 71.9 & 69.7 \\
			&Glister\cite{ref12}   & 74.6 & 73.2 & 70.4 \\

			&EL2N-20\cite{ref11}  & 77.2 & 72.1 & - \\
			&DP\cite{ref1}  & 77.2 & 73.1 & - \\
			\midrule
			\multirow{5}{*}{\begin{sideways}Dynamic\end{sideways}}
			&Random*  & 77.3 & 75.3 & - \\
			&$\epsilon$-greedy\cite{ref6}   & 76.4 & 74.8 & - \\
			&UCB\cite{ref6}  & 77.3 & 75.3 & - \\
			&InfoBatch\cite{ref5}  & 78.2 & 78.1 & 76.5 \\
			&TED(ours)  &\textbf{78.26}\textsubscript{$\pm0.11$}  &\textbf{78.14}\textsubscript{$\pm0.21$} &\textbf{76.83}\textsubscript{$\pm0.17$}  \\
			\midrule
			&Baseline  & \multicolumn{3}{c}{78.2$\pm0.1$} \\				
			\bottomrule
		\end{tabular}
	}
	\end{center}
\end{table}

\textbf{ImageNet-1K.} To explore the performance of TED on large-scale datasets, we conducted experiments on ImageNet-1K using the ResNet-50 model. The results are reported in Table~\ref{table2}, and it is clear from the table that TED continues to lead. It is worth noting that because ImageNet-1K cannot be fully fitted on ResNet-50, none of these methods can achieve lossless performance. This will be a key focus of our future research.

\begin{table}[H]
	\begin{center}
	\caption{TED was evaluated alongside other dynamic pruning methods on ImageNet-1k using ResNet50. "Baseline" refers to the performance without pruning.}
	\label{table2}
		\begin{tabular}{ccc}
			\toprule
			\multicolumn{3}{c}{ImageNet-1k-ResNet50}\\
			\midrule
			Pruning Ratio \% & 30 & 70 \\
			\midrule
			UCB&72.97&68.59\\
			$\epsilon$-greedy&74.10&69.28\\
			Infobatch&74.59&72.51\\

			TED(ours)&\textbf{74.67}\textsubscript{$\pm$0.11}&\textbf{72.68}\textsubscript{$\pm$0.17}\\
			\midrule
			Baseline&\multicolumn{2}{c}{75.28}\\
			\bottomrule
		\end{tabular}
	\end{center}
\end{table}

\begin{figure}[H]
	\centering
	\includegraphics[width=7cm]{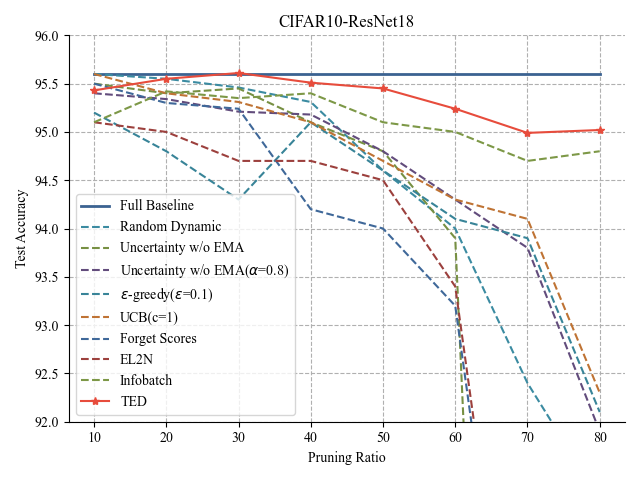}
	\caption{The graph shows how various methods perform on CIFAR-10 using ResNet-18 as the pruning ratio changes.}
	\label{fig5}
\end{figure}
\begin{figure}[H]
	\centering
	\includegraphics[width=6.5cm]{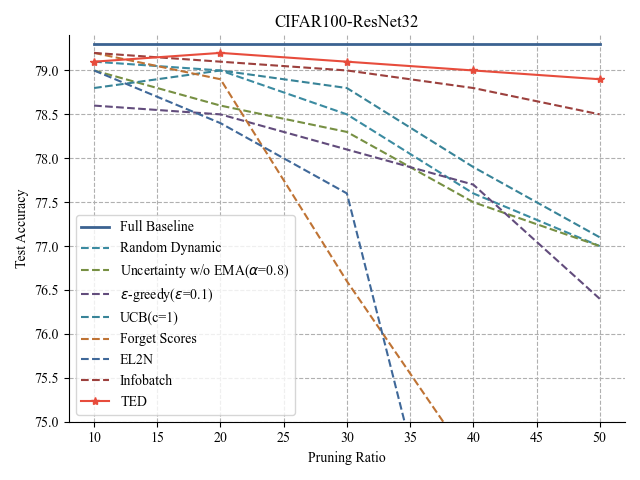}
	\caption{The graph shows how various methods perform on CIFAR-100 using ResNet-32 as the pruning ratio changes.}
	\label{fig6}
\end{figure}

\subsection{Natural Language Understanding Tasks}
In this section, TED's performance in transformer models is examined. Specifically, the BERT-base pre-trained model \cite{ref55} is fine-tuned on GLUE tasks \cite{ref46}. This is the first known exploration of dataset pruning on the BERT-base model. Larger datasets such as MNLI (393k) and QQP (363k) were selected for demonstration. At a 70\% pruning ratio, the results are reported in Table~\ref{table3}. Notably, the evaluation criteria are consistent with \cite{ref53}. More comparative data can be found in Appendix C of the supplementary materials.
\begin{table}[H]
	\begin{center}
	\caption{Comparison of different dynamic pruning methods was conducted using the BERT-base model on GLUE, with a pruning ratio of 70\%. Our evaluation criteria are consistent with \cite{ref53}.}
	\label{table3}
	\resizebox{\columnwidth}{!}{
		\begin{tabular}{l|cccc}
			\toprule
			Dataset& SST-2  & QNLI & MNLI & QQP \\
			\midrule
			Whole Dataset & 92.78  & 89.15 & 84.37 & 91.10 \\
			\midrule
			Infobatch &  92.52  & 89.12 & 82.18 & 91.02 \\
			$\epsilon-$greedy  & 92.63  & 89.54 & 82.53 & 90.68 \\
			UCB  & 92.86  & 88.23 & 79.94 & 89.21 \\
			TED(ours)  & \textbf{93.01}\textsubscript{$\pm0.34$}  &\textbf{90.16}\textsubscript{$\pm0.56$} & \textbf{83.10} \textsubscript{$\pm0.46$}& \textbf{91.24} \textsubscript{$\pm0.29$}\\
			\bottomrule
		\end{tabular}
	}
	\end{center}
\end{table}
\subsection{Fine-tuning Large Language Models}
Recently, large language models have demonstrated remarkable performance \citep{ref50,ref48}. They are typically fine-tuned efficiently for downstream tasks using low-rank adaptation algorithms (LoRA) \cite{ref47} in half-precision (FP16) settings. In our experiments, the model used the LLaMA2-7B model \cite{ref50} and was fine-tuned using the LoRA algorithm, with model evaluation conducted using MMLU \cite{ref51}.
Our experimental results are reported in Table~\ref{table4}. Notably, English instructions generated by GPT-4 \cite{ref48} based on Alpaca prompts were used. Even after removing 30\% of the samples, TED still achieves lossless performance and even slightly outperforms the baseline, indicating that TED is effective in fine-tuning LLMs. This holds practical significance for future LLM research.

\begin{table}[H]
	\begin{center}
	\caption{The fine-tuning results using LoRA on LLaMA2-7B with data generated from English instructions using GPT-4 based on Alpaca prompts, and a pruning ratio of 30\%. "Zero-shot" refers to the pre-trained model without fine-tuning. "Baseline" refers to the performance without pruning.}
	\label{table4}
	\resizebox{\columnwidth}{!}{
		\begin{tabular}{c|ccccc}
			\toprule
			\multirow{2}{*}{Method}&\multicolumn{5}{c}{MMLU}\\
			
			& STEM & Social Sciences & Humanities & Other & AVG. \\ 
			
			\midrule
			Zero-shot & 33.31 & 46.78 & 38.76 & 45.04 & 40.79 \\ 
			Baseline & 35.30 & 47.24 & 41.88 & 50.25 & 43.58 \\ 
			Infobatch   & 34.77 & 47.93 & 40.81 & 49.69 & 43.12\\ 
			\midrule
			
			TED(ours) & 35.14 & 47.44 & 41.95 & 49.90 & 43.60 \\ \bottomrule
	\end{tabular}}
	\end{center}
\end{table}

\subsection{Ablation Experiments}
To further investigate the characteristics and underlying superior performance of TED, we conducted extensive ablation experiments.

\textbf{Optimization Objective.} To verify the performance of the IGD optimization objective, we replaced our scoring function with the current fitting loss, which is also the main evaluation metric of previous work \citep{ref5,ref6,ref27}. Our results are reported in Table~\ref{table5}. The results indicate that the TED method using IGD as the optimization objective performs exceptionally well, suggesting that the IGD optimization objective is distinct from the current fitting loss objective. The current fitting loss focuses solely on the samples that are difficult to learn and repeatedly fits them, while IGD not only helps improve the model's fitting ability on the retained data but also enhances the model's true generalization performance through internal generalization.

\begin{table}[H]
	\begin{center}
	\caption{The performance gap exhibited by TED based on optimization objectives using loss and IGD.}
	\label{table5}
	\resizebox{\columnwidth}{!}{
		\begin{tabular}{ccccc}
			\toprule
			Dataset & \multicolumn{2}{c}{CIFAR10-ResNet18} & \multicolumn{2}{c}{CIFAR100-ResNet32} \\
			\cmidrule(lr){1-1} \cmidrule(lr){2-3} \cmidrule(lr){4-5} 
			
			Pruning Ratio \% & 50 & 70 & 50 & 70 \\
			\cmidrule(lr){1-1} \cmidrule(lr){2-3} \cmidrule(lr){4-5} 
			TED base on Loss & 95.27 & 94.49 & 78.31 & 78.09\\
			TED(ours)&\textbf{95.31}&\textbf{94.94}&\textbf{78.92}&\textbf{78.55}\\
			\bottomrule
		\end{tabular}
	}
	\end{center}
\end{table}

\textbf{Prior effect of IGD.} In Section 3.3, it was verified that IGD possesses a certain degree of prior effect. Here, the impact of this prior effect on overall performance is explored. Similar to other static pruning methods, a trained surrogate model was used to filter data and train the initialized model. The results are reported in Table~\ref{table6}. Surprisingly, TEDS did not exhibit outstanding performance. These findings align with prior research \citep{ref43, ref44}, which suggests that data selected at later stages may not be suitable for the entire training cycle. This indirectly demonstrates that TED is capable of retaining samples that generalize well throughout the training process according to training dynamics. Moreover, it is worth noting that according to other static pruning methods, TEDS leads at high pruning ratios, indicating that IGD optimization objective is very effective.

\textbf{Pruning ratio scheduling.} In dynamic pruning, progressive pruning is believed to align more closely with training dynamics, and a roller-coaster style pruning ratio schedule was designed. In this experiment, it was compared with other commonly used pruning ratio schedules, such as linear decay and cosine decay. Additionally, to compare acceleration efficiency across different computing platforms, save ratio (SR) was used to evaluate acceleration effects. SR is defined as follows, where TFS represents total forward samples and RFS represents realistic forward samples.
\begin{eqnarray}
	SR=\frac{TFS-RFS}{TFS}
\end{eqnarray}

Our results are reported in the Figure~\ref{fig7}. From the figure, it is clear that linear and cosine decay perform similarly to TED in terms of performance, indicating that a progressive pruning approach from high to low pruning ratios can adapt to training dynamics and achieve good performance. As expected, TED's SR outperforms these progressive pruning strategies, suggesting that our roller-coaster pruning ratio schedule can maintain efficient performance while reducing iteration counts to accelerate training. Other pruning ratio scheduling methods can be found in Appendix B of the supplementary materials.
\begin{figure}[H]
	\centering
	\includegraphics[width=6.5cm]{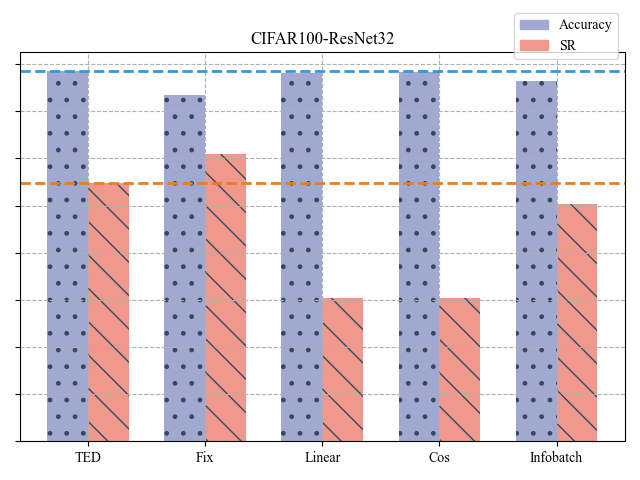}
	\caption{Comparison of accuracy and SR under different pruning ratio schedules: 'Fix' represents fixed sparsity, 'Linear' represents linear decline, 'Cos' represents cosine decline, and Infobatch is another method with uncertain pruning ratios.}
	\label{fig7}
\end{figure}

\begin{table}[H]
	\begin{center}
	\caption{Comparison of performance between TED and TEDS (TED with Surrogate). In TEDS, data selection is performed before training using a pre-trained surrogate model, and then static pruning is applied to the target model. The chosen surrogate model has the same architecture as the target model.}
	\label{table6}
	\resizebox{\columnwidth}{!}{
		\begin{tabular}{ccccc}
			\toprule
			Dataset & \multicolumn{2}{c}{CIFAR10-ResNet18} & \multicolumn{2}{c}{CIFAR100-ResNet32} \\
			\cmidrule(lr){1-1} \cmidrule(lr){2-3} \cmidrule(lr){4-5} 
			
			Pruning Ratio \% & 50 & 70 & 30 & 50 \\
			\cmidrule(lr){1-1} \cmidrule(lr){2-3} \cmidrule(lr){4-5} 
			DP&93.85&90.80&77.11&64.24\\
			EL2N-2&93.21&89.83&77.68&63.63\\
			TEDS & 93.84 & 90.97& 76.97 & 66.94\\
			TED(ours)&\textbf{95.31}&\textbf{94.94}&\textbf{79.11}&\textbf{78.95}\\
			\bottomrule
		\end{tabular}
	}
	\end{center}
\end{table}

\subsection{Generalization Analysis}
To further investigate TED's generalization performance, one-dimensional linear interpolation is used to analyze how TED's data ranking affects the loss landscape. Based on the method proposed in \cite{ref34}, the loss landscape is examined. Previous work \citep{ref35, ref36} suggests that good generalization is found in flat minima.

Specifically, we assess the performance of the model with parameters \(\left(1 - \alpha\right)\theta_1 + \alpha\theta_T\), where \(\theta_T\) is the converged model optimized by different dataset pruning algorithms. As depicted in the Figure~\ref{fig8}, TED's performance is unexpected. In the loss landscape, TED not only remains in a lower range but also exhibits the flattest behavior around \(\alpha\) near 1. In contrast, when we choose a lower IGD score (TED-reverse), the curve presents a landscape opposite to TED.

We have every reason to believe that data helping the model generalize are all within the higher values of IGD. This finding is consistent with the concepts in Section 3.2.
\begin{figure}[H]
	\centering
	\includegraphics[width=7cm]{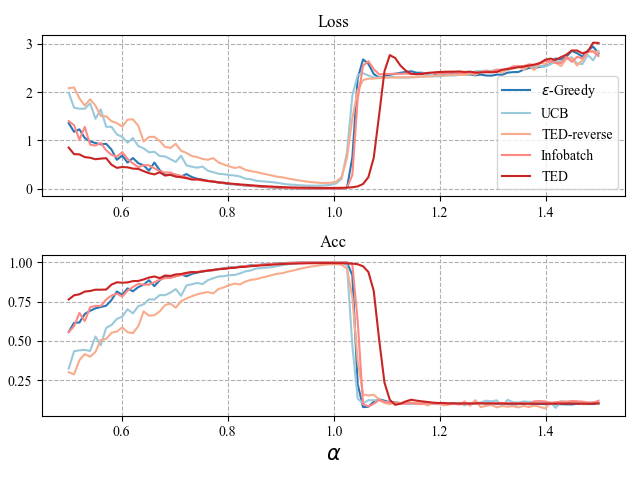}
	\caption{The figure demonstrates the evaluation of optimized models obtained through different pruning methods using one-dimensional linear interpolation.}
	\label{fig8}
\end{figure}


\section{Conclusion}
In this work, we proposed the TED pruning method, which is based on the internal generalization of dataset pruning and measures sample importance using the internal generalization distance before and after pruning. This metric not only captures essential data within the entire dataset but also achieves implicit regularization. Additionally, we found progressive pruning based on training dynamics, where using pruning ratios that increase from low to high better aligns with the model's training dynamics. For this, we designed a roller-coaster style pruning ratio schedule. According to experiments, TED achieved notable success in image classification, natural language processing, and fine-tuning large language models.

\textbf{Limitations and future works.} Although TED does not require a surrogate model, the TED method relies on identifying redundant samples during training. Similar to the lottery ticket hypothesis \cite{ref54} in model pruning, the ability to find suitable subsets for early training or for untrained models is our next goal.


\bibliography{mybibfile}

\newpage
\onecolumn

\appendix

\section{Implementation Details}

In this section, we primarily introduce the implementation details of the experiments. It is important to note that the TED pruning method has several hyperparameters: the Polyak averaging coefficient \(\alpha\), the hyperparameter \(\beta\) for the pruning rate schedule, and the annealing hyperparameter \(\gamma\) (the proportion of total training time for the annealing step). In all experiments, \(\alpha\) was set to 0.9, and \(\beta\) was set to 0.2-0.3. For image classification tasks, \(\gamma\) was set to 0.125, while for natural language understanding and large language model fine-tuning tasks, \(\gamma\) was set to 0.
\subsection{Implementation Details for Image Classification}
CIFAR-10 and CIFAR-100 are datasets containing 50,000 samples each. We used ResNet-18 and ResNet-32 models for our experiments, and our hyperparameters are reported in Table~\ref{table1}.

ImageNet-1K consists of nearly 1.2 million samples across 1,000 classes. We used ResNet-50 for the experiments, with specific parameters reported in Table~\ref{table2}. For learning rate scheduling, we multiply the learning rate by 0.1 every 30 epochs.

\begin{table}[h]
	\begin{center}
	{\caption{The hyperparameters we used in the CIFAR experiments.}}
	\label{table1}
	
	\begin{tabular}{|c|c|c|c|c|c|c|c|}
		\hline
		Epoch & lr &batch size&momentum&weight decay&optimizer&max lr&lr scheduler \\
		\hline
		200&0.2&128&0.9&5e-4&lars\cite{ref2}&5.2&OneCycle\cite{ref1}\\
		\hline
	\end{tabular}
	\end{center}
\end{table}

\begin{table}[h]
	\begin{center}
	\caption{The hyperparameters we used in the ImageNet-1K experiments.}
	\label{table2}
	
	\begin{tabular}{|c|c|c|c|c|c|}
		\hline
		Epoch & lr &batch size&momentum&weight decay&optimizer \\
		\hline
		90&0.1&256&0.9&1e-4&SGD\\
		\hline
	\end{tabular}
	\end{center}
\end{table}

\subsection{Implementation Details for NLU}
The GLUE benchmark is a suite designed to evaluate the performance of natural language understanding (NLU) systems. It comprises a set of different types of language understanding tasks to measure a model's performance across a broad range of tasks. These tasks encompass typical challenges in natural language understanding applications.

The pre-trained BERT-base \cite{ref6} model provided by Hugging Face was used by us. The specific fine-tuning hyperparameters are presented in Table~\ref{table3}.

\begin{table}[h]
	\begin{center}
	\caption{The hyperparameters we used in the GLUE experiments.}
	\label{table3}
	
	\begin{tabular}{|c|c|c|c|c|}
		\hline
		Epoch & lr &batch size&optimizer &lr scheduler\\
		\hline
		10&2e-5&32& Adam\cite{ref7}&linear\\
		\hline
	\end{tabular}
	\end{center}
\end{table}

\subsection{Implementation Details for Fine-tuning LLM}
In our experiments, the LLaMA2-7B \cite{ref4} pre-trained model, which has 7 billion parameters and is suitable for natural language processing tasks such as dialogue generation, question answering, and text summarization, was used by us. Additionally, the alpaca-gpt4 data from Hugging Face \cite{ref5}, consisting of English instructions generated by GPT-4 \cite{ref3} based on Alpaca prompts and containing 52k samples, was used by us. Our fine-tuning hyperparameters are shown in Table~\ref{table4}.

\begin{table}[H]
	\begin{center}
	\caption{The hyperparameters we used in the fine-tuning LLM.}
	\label{table4}
	\begin{tabular}{|c|c|c|c|c|c|}
		\hline
		Epoch & lr &batch size&optimizer &warmup steps&lr scheduler\\
		\hline
		3&5e-5&16& Adam\cite{ref7}&20&Cosine\\
		\hline
	\end{tabular}
	\end{center}
\end{table}

\section{Pruning Rate Scheduling}
\subsection{Roller-coaster schedule}
In section 3.2 of the main text, we designed a roller-coaster pruning rate schedule. As shown in the Figure~\ref{fig1}, the curve resembles the descent of a roller coaster track, with the rate of decline slowing down as training progresses. This curve is influenced by \(\beta\); simply put, the smaller the value of \(\beta\), the faster the rate of decline. The Figure~\ref{fig2} also presents other scheduling methods.

\begin{figure}[h]
	\centering
	\includegraphics[width=8cm]{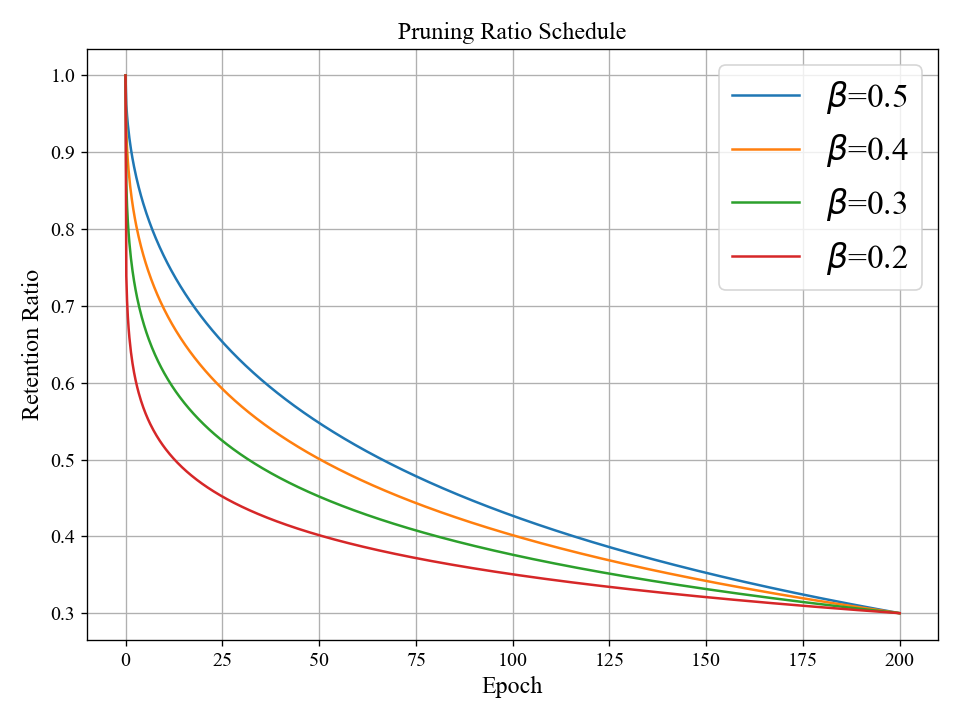}
	\caption{The figure shows the roller-coaster schedule curves at a 70\% pruning rate for different \(\beta\) values.}
	\label{fig1}
\end{figure}
\begin{figure}[tbh]
	\centering
	\includegraphics[width=8cm]{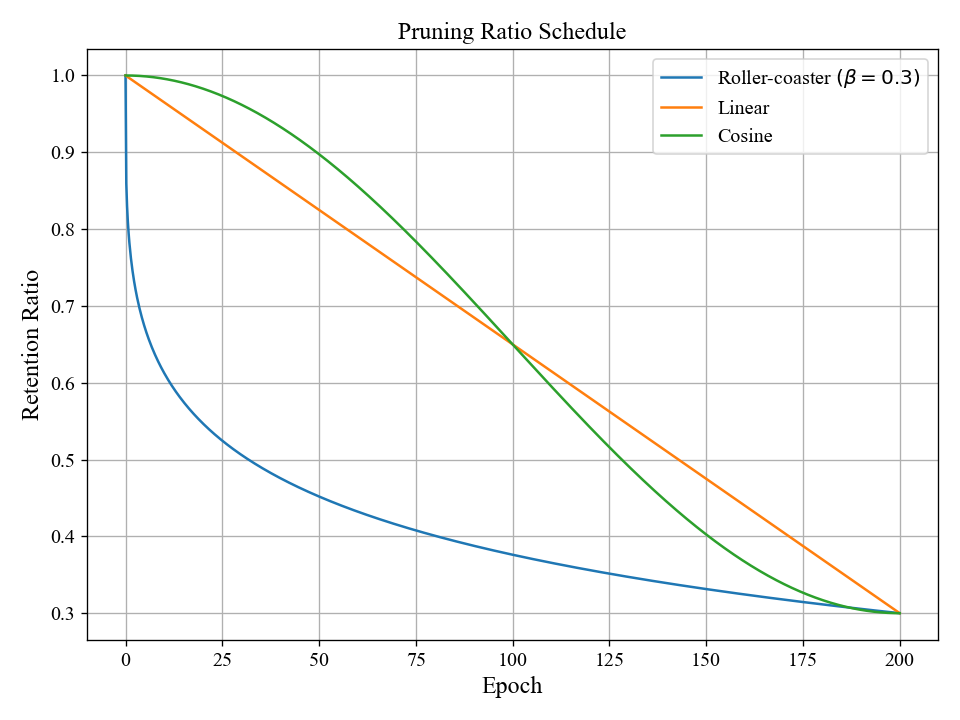}
	\caption{The curves of the roller-coaster schedule and other schedules at a 70\% pruning ratio.}
	\label{fig2}
\end{figure}
\section{The Complete Results on GLUE}
In the GLUE experiments, only the larger datasets were initially showcased by us. In this section, TED's performance across the entire set of GLUE tasks under different pruning rates is presented.
\begin{table}[H]
	\begin{center}
	\caption{Comparison of Dynamic Pruning Methods on GLUE}
	\label{table5}
	\begin{tabular}{clcccccccc}
		\toprule
		&Dataset & RTE & CoLA & SST-2 & STS-B & MRPC & QNLI & MNLI & QQP \\
		&Eval Metric&Acc&Matthew's Cor.&Acc&Pearson Cor.&Acc&Matched Acc.&Acc&Acc\\
		\midrule
		&Whole dataset & 67.08 & 57.61 & 92.78 & 88.76 & 86.24 & 89.15 & 84.37 & 91.10 \\
		\midrule
		\multirow{4}{*}{\begin{tabular}[c]{@{}c@{}}30\%\end{tabular}}
		&Infobatch & \textbf{67.06} & 59.08 & 92.89 & 88.26 & 84.38 & 91.26 & \textbf{84.40} & 91.32 \\
		&Greedy & 64.78 & 58.55 & 93.30 & 88.59 & 84.94 & 90.89 & 84.31 & 91.44 \\
		&UCB & 64.33 & 57.41 & 93.08 & \textbf{88.63} & 84.78 & 90.14 & 84.25 & 91.54 \\
		&TED(ours) & 66.21 & \textbf{59.98} & \textbf{93.08} & 88.49 & \textbf{86.16} & \textbf{91.34} & 84.27 & \textbf{91.68} \\
		\midrule
		\multirow{4}{*}{\begin{tabular}[c]{@{}c@{}}50\%\end{tabular}}
		&Infobatch & 68.21 & 58.12 & 92.86 & 88.85 & 84.94 & 90.55 & \textbf{84.24} & \textbf{91.35} \\
		&Greedy & 68.17 & \textbf{60.39} & \textbf{93.31} & 88.26 & 85.33 & 90.11 & 84.14 & 91.25 \\
		&UCB & \textbf{68.88} & 57.61 & 93.98 & 88.10 & 82.61 & 88.64 & 84.17 & 90.69 \\
		&TED(ours) & 68.22 & 58.58 & 92.67 & \textbf{88.94} & \textbf{85.81} & \textbf{91.18} & 83.96 & 91.29 \\
		\midrule
		\multirow{4}{*}{\begin{tabular}[c]{@{}c@{}}70\%\end{tabular}}
		&Infobatch & 64.67 & \textbf{58.21} & 92.52 & 87.38 & 81.65 & 89.12 & 82.18 & 91.02 \\
		&Greedy & 64.94 & 56.59 & 92.63 &\textbf{88.49} & 85.09 & 89.54 & 82.53 & 90.68 \\
		&UCB & 63.66 & 55.69 & 92.86 & 87.25 & 45.83 & 88.23 & 79.94 & 89.21 \\
		&TED(ours) & \textbf{65.45} & 57.33 & \textbf{93.01} & 87.87 & \textbf{85.69} & \textbf{90.16} & \textbf{83.10} & \textbf{91.24} \\
		\bottomrule
	\end{tabular}
	\end{center}
\end{table}
\end{document}